\documentclass[sigconf]{acmart}

\usepackage{graphicx}
\usepackage[
font=Large
]{subfig}
\usepackage{tikz}
\usetikzlibrary{shapes, arrows,petri,topaths,backgrounds,patterns,positioning}
\usepackage{pgfplots}
\usepgfplotslibrary{statistics}
\usepgfplotslibrary{groupplots}
\usepackage{units}
\usepackage{booktabs}
\usepackage{tabularx}
\usepackage{algorithmic}

\usepackage[linesnumbered,lined,boxed]{algorithm2e}

\usepackage{xcolor,pifont}

\usepackage{multirow}
\usepackage{amsmath,bm}

\definecolor{cycle2}{RGB}{106, 191, 0}
\definecolor{cycle3}{RGB}{191, 0, 0}
\definecolor{amber}{rgb}{1.0, 0.75, 0.0}

\newcommand{\cmark}{\textcolor{cycle2}{\ding{52}}} %
\newcommand{\xmark}{\textcolor{cycle3}{\ding{56}}}

\definecolor{awesome}{rgb}{1.0, 0.13, 0.32}
\definecolor{ao(english)}{rgb}{0.0, 0.5, 0.0}
\pgfdeclarelayer{background}
\pgfdeclarelayer{foreground}
\pgfsetlayers{background,main,foreground}

\tikzstyle{scoring}=[draw, fill=blue!20, text width=11em, 
text centered,rounded corners, minimum height=2.5em]

\tikzstyle{dp}=[draw, fill=red!20, text width=2em, 
text centered,rounded corners, minimum height=2.5em]

\tikzstyle{score}=[draw, fill=red!20, text width=8em, 
text centered,rounded corners, minimum height=2.5em]

\tikzstyle{score2}=[draw, fill=red!20, text width=6.3em, 
text centered,rounded corners, minimum height=2.5em]

\tikzstyle{outcome}=[draw, fill=green!20, text width=6em, 
text centered,rounded corners, minimum height=2.5em]

\SetAlCapSkip{1em}

\usepgfplotslibrary{fillbetween}

\AtBeginDocument{%
	\providecommand\BibTeX{{%
			\normalfont B\kern-0.5em{\scshape i\kern-0.25em b}\kern-0.8em\TeX}}}

\setcopyright{acmcopyright}
\copyrightyear{2018}
\acmYear{2018}
\acmDOI{10.1145/1122445.1122456}

\acmConference[Woodstock '18]{Woodstock '18: ACM Symposium on Neural
	Gaze Detection}{June 03--05, 2018}{Woodstock, NY}
\acmBooktitle{Woodstock '18: ACM Symposium on Neural Gaze Detection,
	June 03--05, 2018, Woodstock, NY}
\acmPrice{15.00}
\acmISBN{978-1-4503-XXXX-X/18/06}

\newtheorem{definition}{Definition}
\newtheorem{theorem}{Theorem}
\newtheorem{axiom}{Axiom}
\newtheorem{lemma}{Lemma}
\usepgfplotslibrary{groupplots}
\usetikzlibrary{pgfplots.groupplots}

\usetikzlibrary{
	pgfplots.colorbrewer,
}
\pgfplotsset{
	cycle list/.define={my marks}{
		every mark/.append style={solid,fill=\pgfkeysvalueof{/pgfplots/mark list fill}},very thick\\
		every mark/.append style={solid,fill=\pgfkeysvalueof{/pgfplots/mark list fill}},very thick\\
		every mark/.append style={solid,fill=\pgfkeysvalueof{/pgfplots/mark list fill}},very thick\\
		every mark/.append style={solid,fill=\pgfkeysvalueof{/pgfplots/mark list fill}},very thick\\
	},
}

\begin{document}

	\title{The Shapley Value of Classifiers in Ensemble Games}
	
	\author{Benedek Rozemberczki}
	\affiliation{%
		\institution{The University of Edinburgh}
		\city{Edinburgh}
		\country{United Kingdom}}
	\email{benedek.rozemberczki@ed.ac.uk}
	
	\author{Rik Sarkar}
	\affiliation{%
		\institution{The University of Edinburgh}
		\city{Edinburgh}
		\country{United Kingdom}}
	\email{rsarkar@inf.ed.ac.uk}
	

\begin{abstract}
What is the value of an individual model in an ensemble of binary classifiers? We answer this question by introducing a class of transferable utility cooperative games called \textit{ensemble games}. In machine learning ensembles, pre-trained models cooperate to make classification decisions. To quantify the importance of models in these ensemble games, we define \textit{Troupe} -- an efficient algorithm which allocates payoffs based on approximate Shapley values of the classifiers. We argue that the Shapley value of models in these games is an effective decision metric for choosing a high performing subset of models from the ensemble. Our analytical findings prove that our Shapley value estimation scheme is precise and scalable; its performance increases with size of the dataset and ensemble. Empirical results on real world graph classification tasks demonstrate that our algorithm produces high quality estimates of the Shapley value. We find that Shapley values can be utilized for ensemble pruning, and that adversarial models receive a low valuation. Complex classifiers are frequently found to be responsible for both correct and incorrect classification decisions. 
\end{abstract}

\maketitle

\section{Introduction}\label{sec:shapley_introduction}
The advent of black box machine learning models raised fundamental questions about how input features and individual training data points contribute to the decisions of expert systems \cite{shap, data_shapley_equitable}. There has also been interest in how the heterogeneity of models in an ensemble results in heterogeneous contributions of those to the classification decisions of the ensemble \cite{decision_ensemble,tumer1996error}. For example one would assume that computer vision, credit scoring and fraud detection systems which were trained on varying quality proprietary datasets output labels for data points with varying accuracy. Another source of varying model performance can be the complexity of models e.g. the number of weights in a neural network or the depth of a classification tree. 

 Quantifying the contributions of models to an ensemble is paramount for practical reasons. Given the model valuations, the gains of the task can be attributed to specific models, large ensembles can be reduced to smaller ones without losing accuracy \cite{pruning_1, pruning_2} and performance heterogeneity of ensembles can be gauged \cite{decision_ensemble}. This raises the natural question: How can we measure the contributions of models to the decisions of the ensemble in an efficient, model type agnostic, axiomatic and data driven manner?

\begin{figure}[h!]
\centering
\includegraphics[scale=0.105]{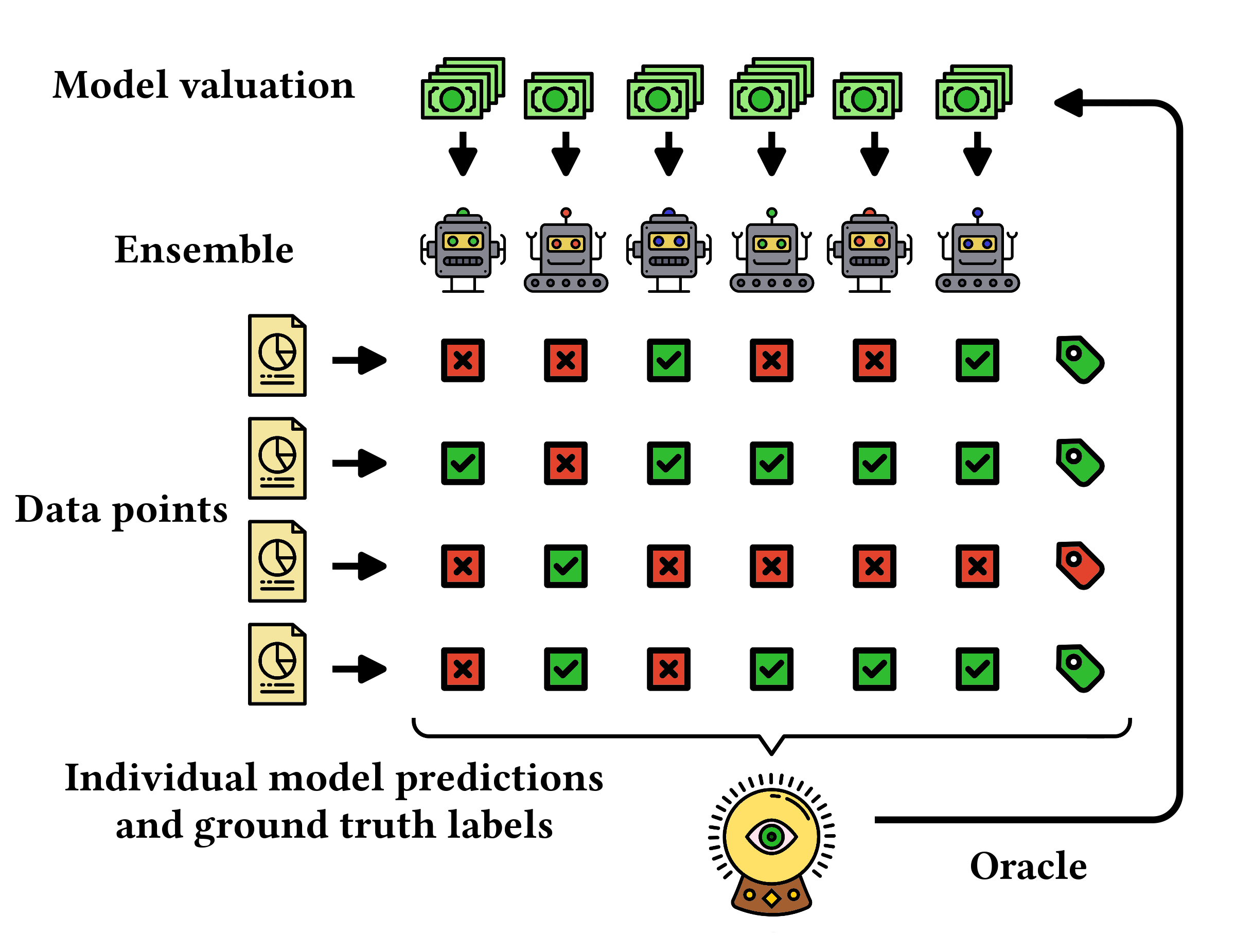}
\caption{An overview of the model valuation problem. Models in the ensemble receive a set of data points and score those. Using the predictions and ground truth labels the oracle quantifies the worth of models.}\label{fig:shapley_eyecandy}
\end{figure}

We frame this question as one of valuation of models in an ensemble. The solution to the problem requires an analytical framework to assess the worth of individual classifiers in the ensemble. This idea is described in Figure \ref{fig:shapley_eyecandy}. Each classifier in the ensemble receives the data points, and they output for each data point a probability distribution over the potential classes. Using these propensities an oracle -- which has access to the ground truth -- quantifies the worth of models in the ensemble. These importance metrics can be used to make decisions -- e.g. pruning the ensemble and allocation of payoffs. 

\textbf{Present work.} We introduce \textit{ensemble games}, a class of transferable utility cooperative games \cite{tu_games}. In these games binary classifiers which form an ensemble play a voting game to assign a binary label to a data point by utilizing the features of the data point. Building on the ensemble games we derive \textit{dual ensemble games} in which the classifiers cooperate in order to misclassify a data point. We do this to characterize the role of models in incorrect decisions.

We argue that the Shapley value \cite{shapley1953value}, a solution concept from cooperative game theory, is a model importance metric. The Shapley value of a classifier in the ensemble game defined for a data point can be interpreted as the probability of the model becoming the pivotal voter in a uniformly sampled random permutation of classifiers. Computing the exact Shapley values in an ensemble game would take factorial time in the number of classifiers. In order to alleviate this we exploit an accurate approximation algorithm of the individual Shapley values which was tailored to voting games \cite{fatima2008linear}. We propose \textit{Troupe}, an algorithm which approximates the average of Shapley values in ensemble games and dual games using data. We utilize the average Shapley values as measures of model importance in the ensemble and discuss that the Shapley values are interpretable as an importance distribution over classifiers.

We evaluate \textit{Troupe} by performing various classification tasks. Using data from real world webgraphs (Reddit, GitHub, Twitch) we demonstrate that \textit{Troupe} outputs high quality estimates of the Shapley value. We validate that the Shapley value estimates of \textit{Troupe} can be used as a decision metric to build space efficient and accurate ensembles. Our results establish that more complex models in an ensemble have a prime role in both correct and incorrect decisions.

\textbf{Main contributions.} Specifically the contributions of our work can be summarized as:
\begin{enumerate}
    \item We propose ensemble games and their dual games to model the contribution of individual classifiers to the decisions of voting based ensembles.
    \item We design \textit{Troupe} an approximate Shapley value based algorithm to quantify the role of classifiers in decisions. 
    \item We provide a probabilistic bound for the approximation error of average Shapley values estimated from labeled data. 
    \item We empirically evaluate \textit{Troupe} for model valuation and forward ensemble building on graph classification tasks.
\end{enumerate}
The rest of this work has the following structure. In Section \ref{sec:shapley_related_work} we discuss related work on the Shapley value, its approximations and applications in machine learning. We introduce the concept of ensemble games in Section \ref{sec:shapley_theory} and discuss Shapley value based model valuation in Section \ref{sec:shapley_algorithm} with theoretical results. We evaluate the proposed algorithm experimentally in Section \ref{sec:shapley_experiments}. We summarize our findings in Section \ref{sec:shapley_conclusions} and discuss future work. The reference implementation of \textit{Troupe} is available at \url{https://github.com/username/repositoryname} [Anonymized for double blind review].
\section{Related Work}\label{sec:shapley_related_work}
Our model pruning framework intersects with research about the Shapley value and existing approaches to ensemble pruning.
\subsection{The Shapley value}
The \textit{Shapley} value \cite{shapley1953value} is a solution to the problem of distributing the gains among players in a transferable utility cooperative game \cite{tu_games}. It is widely known for its desirable axiomatic properties \cite{chalkiadakis2011computational} such as \textit{efficiency} and linearity. However, exact computation of Shapley value takes factorial time, making it intractable in games with a large number of players. General \cite{maleki,shapley_mle} and game  specific \cite{fatima2008linear} approximation techniques have been proposed. In Table \ref{tab:shapley_comparison}, we compare various approximation schemes with respect to certain desired properties.

Shapley values can be approximated using a Monte Marlo {\em MC} sampling of the permutations of players and also with a truncated Monte Carl sampling variant {\em TMC}~\cite{shapley_monte_carlo,shapley_permutation,maleki}. 
A more tractable approximation is proposed in  \cite{shapley_mle}, using a multilinear extension (\textit{MLE}) of the Shapley value. A variant of this technique \cite{shapley_mmle_1,shapley_mmle_2} calculates the value of large players explicitly and applies the \textit{MLE} technique to small ones. The only approximation technique tailored to weighted voting games is the expected marginal contributions method (\textit{EMC}) which estimates the Shapley values based on contributions to varying size coalitions. Our proposed algorithm \textit{Troupe} builds on \textit{EMC}. 

\begin{table}[h!]
\centering
\caption[Comparison of Shapley value computation and approximation techniques.]{Comparison of Shapley value computation and approximation techniques in terms of having (\cmark) and missing (\xmark) desiderata; complexities with respect to the number of players $m$ and permutations $p$.}\label{tab:shapley_comparison}
{\small\begin{tabular}{cccccc}
\hline
\textbf{Method} & \textbf{Voting}           & \textbf{Bound} & \textbf{Non-Random} & \textbf{Space} & \textbf{Time} \\
\hline
Explicit             &\cmark&         \cmark    &     \cmark          &  $\mathcal{O}(m)$     &  $\mathcal{O}(m!)$    \\
MC \cite{shapley_monte_carlo}    &\xmark&      \cmark       &      \xmark         &  $\mathcal{O}(m)$       &    $\mathcal{O}(mp)$    \\
TMC \cite{data_shapley_equitable}    &\xmark&      \cmark       &      \xmark         &  $\mathcal{O}(m)$       &    $\mathcal{O}(mp)$    \\
 MLE   \cite{shapley_mle}&\xmark&       \xmark     &          \cmark     &  $\mathcal{O}(m)$    &    $\mathcal{O}(m)$ \\
  MMLE \cite{shapley_mmle_1,shapley_mmle_2}  &\xmark&      \xmark      &        \cmark       & $\mathcal{O}(m)$     &  $\mathcal{O}(m!)$ \\
EMC \cite{fatima2008linear} &\cmark& \cmark            &         \cmark      &   $\mathcal{O}(m)$      & $\mathcal{O}(m^2)$       \\
\hline
\end{tabular}}
\end{table}

Shapley value has previously been used in machine learning for measuring feature importance \cite{lipovetsky2001analysis, pinter2011regression, mokdad2015determination,frye2020asymmetric,liu2020shapley,chen2018shapley,labreuche2018explaining}. In the feature selection setting the features are seen as players that cooperate to achieve high goodness of fit. Various discussed approximation schemes \cite{maleki,cohen2005feature,feature_selection_permutation} have been exploited to make feature importance quantification in high dimensional spaces feasible \cite{sun2012using, sundararajan2020many, shap,covert2021improving,williamson2020efficient} when explicit computation is not tractable. Another machine learning domain for applying the Shapley value was the pruning of neural networks \cite{propagation, neuronshapley,stier2018analysing}. In this context approximate Shapley values of hidden layer neurons are used to downsize overparametrized classifiers. It is argued in \cite{stier2018analysing} that pruning neurons  is analogous to feature selection on hidden layer features. Finally, there has been increasing interest in the equitable valuation of data points with game theoretic tools \cite{data_shapley, data_shapley_equitable}. In such settings the estimated Shapley values are used to gauge the influence of individual points on a supervised model. These approximate scores are obtained with group testing of features \cite{data_shapley} and permutation sampling \cite{permutation_mc, data_shapley_equitable}.

\subsection{Ensemble pruning and building}

We compare Troupe to existing ensemble pruning and bulding approaches in terms of desired properties in Table \ref{tab:troupe_comparison}. As one can see our framework is the only one which has all of these characteristics. Specifically, these properties are:
\begin{itemize}
    \item \textit{Agnostic:} The procedure can prune/build ensembles of heterogeneous model types, not just a specific type (e.g. trees).
    \item \textit{Set based:} An algorithm is able to return a subset of models with a high-performance, not just a single one.
    \item \textit{Diverse:} The value of individual models is implicitly or explicitly affected by how diverse their predictions are.
    \item \textit{Bidirectional:} A bidirectional ensemble selection technique can select a sub-ensemble in a forward or backward manner.
\end{itemize}
\begin{table}[h!]
\centering
\caption[Comparison of Shapley value computation and approximation techniques.]{Comparison of ensemble pruning and building techniques in terms of having (\cmark) and missing (\xmark) desiderata.}\label{tab:troupe_comparison}
{\small\begin{tabular}{ccccc}
\hline
\textbf{Method} & \textbf{Agnostic}           & \textbf{Set based} & \textbf{Diverse} & \textbf{Bidirectional} \\
\hline
Greedy       \cite{caruana2004ensemble}     &\cmark&     \cmark      &      \xmark      &  \cmark      \\
WV-LP       \cite{zhang2011sparse}     &\cmark&      \cmark     &       \xmark      & \xmark       \\
RE         \cite{margineantu1997pruning}    &\cmark&       \cmark    &    \xmark        &   \xmark     \\
DREP   \cite{li2012diversity}          &\cmark&        \cmark   &       \cmark     &     \xmark   \\
COMEP \cite{bian2019ensemble}         &\xmark&  \cmark         &     \cmark       &    \xmark    \\
EP-SDP \cite{zhang2006ensemble}           &\cmark&     \cmark      &    \xmark        &   \xmark     \\
CART SySM           \cite{weinberg2019selecting} &\xmark&    \xmark       &   \cmark         & \xmark       \\
Troupe (ours)           &\cmark&         \cmark    &     \cmark          &  \cmark   \\

\hline
\end{tabular}}
\end{table}
\section{Ensembles games}\label{sec:shapley_theory}
We now introduce a novel class of co-operative games and examine axiomatic properties of solution concepts which can be applied to these games. We will discuss shapley value as an exact solution for these games, \cite{shapley1953value}, and discuss approximations of the Shapley value based solution \cite{shapley_mle, shapley_monte_carlo, fatima2008linear}.

\subsection{Ensemble game}

We define an ensemble game to be one where binary classifier models (players) co-operate to label a single data point. The aggregated decision of the ensemble is assumed to be made by a vote of the players. We assume that a set of labelled data points are known: 
\begin{definition} \textit{\textbf{Labeled data point.}} Let $(\textbf{x},y)$ be a labeled data point where $\textbf{x}\in \mathbb{R}^d$ is the feature vector and $y\in \left \{0,1\right \}$ is the corresponding binary label.
\end{definition}

Our work considers arbitrary binary classifier models (e.g. classification trees, support vector machines, neural networks) that operate on the same input feature vector $\textbf{x}$. This approach is agnostic of the exact type of the model, we only assume that $M$ can output a probability of the data point having a positive label. The model owner does does not access the label, just a probability of $y=1$ is output by the model.

\begin{definition} \textit{\textbf{Positive classification probability for model $M$.}}
Let $(\textbf{x},y)$ be a labeled data point and $M$ be a binary classifier, $P(y=1\mid M,\textbf{x})$ is the probability of the data point having a positive label output by classifier $M$.
\end{definition}

We are now ready to define the operation of an ensemble classifier consisting of multiple models. 
\begin{definition} \textit{\textbf{Ensemble.}}
An ensemble is a set $\mathcal{M}$ of size $m$ that consists of binary classifier models $M\in \mathcal{M}$ which can each output a probability for a data point $\textbf{x}\in \mathbb{R}^d$ having a positive label. The ensemble sets the probability of a label as: 
$$P(y=1\mid \mathcal{M},\textbf{x})= \sum _{M \in \mathcal{M}} P(y=1\mid M,\textbf{x})/m.$$
And it makes decision about the label as: 
$$\widehat{y}=\begin{cases} 1& \text{ if } P(y=1\mid \mathcal{M},\textbf{x})\geq \gamma,\\ 
 0& \text{otherwise}.
\end{cases}$$
where $0\leq \gamma\leq 1$ is the decision threshold and $\widehat{y}$ is the predicted label of the data point.
\end{definition}

\begin{definition} \textit{\textbf{Sub-ensemble.}}
A sub-ensemble $\mathcal{S}$ is a subset $\mathcal{S}\subseteq \mathcal{M}$ of binary classifier models.
\end{definition}

\begin{definition} \textit{\textbf{Individual model weight.}}
The individual weight of the vote for $M$, in sub-ensemble $\mathcal{S}\subseteq \mathcal{M}$ for data point $(y,\textbf{x})$ is defined as:
$$w_{M}=\begin{cases}
P(y=1\mid M,\textbf{x})/m& \text{ if } y = 1,\\ 
P(y=0\mid M,\textbf{x})/m& \text{otherwise}.
\end{cases}$$
\end{definition}
Under this definition. the individual model weight of any binary classifier $M\in \mathcal{M}$ is bounded: $0\leq w_{M}\leq 1/m$. Note that the weight of $M \in \mathcal{S}$ depends on $m$ -- the size of the larger ensemble and not on the size $|\mathcal{S}|$ of the sub-ensemble. 

\begin{definition} \textit{\textbf{Ensemble game.}}
Let $\mathcal{M}$ be a set of binary classifiers. An ensemble game for a labeled data point $(y, \textbf{x})$ is then a co-operative game $G= \left (\mathcal{M},v\right )$ in which:
$$v(\mathcal{S})=\begin{cases}
 1& \text{ if } w(\mathcal{S})\geq \gamma,\\ 
 0& \text{otherwise}. 
\end{cases}$$
where $w(\mathcal{S}) =\sum_{M \in \mathcal{S}}w_{M}$
for any sub-ensemble $\mathcal{S}\subseteq \mathcal{M}$ and threshold $0\leq \gamma\leq 1$.
\end{definition}
This definition is the central idea in our work. The models in the ensemble play a cooperative voting game to classify the data point correctly. When the data point is classified correctly the payoff is 1, an incorrect classification results in a payoff of 0. Each model casts a weighted vote about the data point and our goal is going to be to quantify the value of individual models in the final decision. In other words, we would like to measure how individual binary classifiers \textit{contribute on average} to the correct classification of a specific data point. This {\em solution concept is described in the next section(\ref{sec:solution-concept})}. 

We can consider a misclassification as a dual ensemble game: 
\begin{definition} \textbf{Dual ensemble game.}
Let $\mathcal{M}$ be a set of binary classifiers. A dual ensemble game for a labeled data point $(y, \textbf{x})$ is then a co-operative game $G= \left (\mathcal{M},\widetilde{v}\right )$ in which:
$$\widetilde{v}(\mathcal{S})=\begin{cases}
 1& \text{ if } \widetilde{w}(\mathcal{S})\geq \widetilde{\gamma},\\ 
 0& \text{otherwise}. 
\end{cases}$$
for a binary classifier ensemble vote score $ 0\leq \widetilde{w}(\mathcal{S}) \leq 1$  where $\widetilde{w}(\mathcal{S}) =\sum_{M \in \mathcal{S}}(1/m-w_{M})$
for any sub-ensemble $\mathcal{S}\subseteq \mathcal{M}$ and inverse cutoff value $0\leq \widetilde{\gamma}\leq 1$ defined by $\widetilde{\gamma} = 1-\gamma$.
\end{definition}
If the sum of classification weights for the binary classifiers is below the cutoff value the models in the ensemble misclassify the point, lose the ensemble game and as a consequence receive a payoff that is zero. In such scenarios it is interesting to ask: how can we describe the role of models in the misclassification? The dual ensemble game is derived from the original ensemble game in order to characterize this situation. 

The classification game and its dual can be reframed simply as: 

\begin{definition} \textbf{Simplified ensemble game.} An ensemble game in simplified form is described by the cutoff value -- weight-vector tuple $(\gamma,[w_1,\dots,w_m]).$
\end{definition}

\begin{definition} \textbf{Simplified dual ensemble  game.} Given a simplified form ensemble game $(\gamma,[w_1,\dots,w_m])$, the corresponding simplified dual ensemble game is defined by the cutoff value -- weight vector tuple:
\begin{align*}
    (\widetilde{\gamma},[\widetilde{w}_1,\dots,\widetilde{w}_m])=(1-\gamma,[1/m-w_1,\dots,1/m-w_m])
\end{align*}
\end{definition}
The simplified forms of ensemble and dual ensemble games are compact data structures which can describe the game without the models themselves and the enumeration of every sub-ensemble.

\subsection{Solution concepts for model valuation}\label{sec:solution-concept}

We have defined the binary classification problem with an ensemble as a weighted voting game, which is a type of co-operative game. Now we will argue that \textit{solution concepts} of co-operative games are suitable for the valuation of individual models which form the binary classifier ensemble.

\begin{definition} \textit{\textbf{Solution concept.}} A solution concept defined for the ensemble game $G=(\mathcal{M},v)$ is a function which assigns the real value $\Phi_M(\mathcal{M},v) \in \mathbb{R}$ to each binary classifier $M\in\mathcal{M}$.
\end{definition}
The scalar $\Phi_M$ can be interpreted as the value of the individual binary classifier $M$ in the ensemble $\mathcal{M}$. In the following we discuss axiomatic properties of solution concepts which are the desiderata for model valuation functions, 
and the implications of the axioms 
in the context of model valuation in binary ensemble games.-edited 

\begin{axiom}\label{shapley:ax_1} \textit{\textbf{Null classifier.}} A solution concept has the null classifier property if $\forall\,\, \mathcal{S}\subseteq \mathcal{M}: v(\mathcal{S}\cup \left\{ M\right\})=v(\mathcal{S})$ then $ \Phi_M(\mathcal{M},v)=0$. 
\end{axiom}
Having the null classifier property means that a binary classifier which always has a zero marginal contribution in any sub-ensemble will have a zero payoff on its own. This also implies that the classifier never casts the deciding vote to correctly classify the data point when it is added to a sub-ensemble. Conversely, in the dual ensemble game the model never contributes to the misclassification of the data point.
\begin{axiom}\label{shapley:ax_2} \textit{\textbf{Efficiency.}} A solution concept satisfies the efficiency property if $v(\mathcal{M})=\sum_{M\in \mathcal{M}}\Phi_M(\mathcal{M},v)$.
\end{axiom}

That is, the value (loss or gain) of an ensemble can be split precisely into the contributed value of the constituent models.

\begin{axiom}\label{shapley:ax_3} \textit{\textbf{Symmetry.}} A solution concept has the symmetry property if $\forall \mathcal{S} \subseteq \mathcal {M}\setminus \left\{M',M''\right\}:v(\mathcal{S}\cup \left \{ M'\right \})=v(\mathcal{S}\cup \left \{ M''\right \})$ implies that $\Phi_{M'}(\mathcal{M},v)=\Phi_{M''}(\mathcal{M},v)$.
\end{axiom}
Two binary classifiers which make equal marginal contribution to all sub-ensembles have the same value in the full ensemble. 


\begin{axiom}\label{shapley:ax_4} \textit{\textbf{Linearity.}} A solution concept has the linearity property if given any two ensemble games $G=(\mathcal{M},v)$ and $G'=(\mathcal{M},v')$ on the same set $\mathcal{M}$, the binary classifier $M$ satisfies $\Phi_M(\mathcal{M},v+v')=\Phi_M(\mathcal{M},v)+\Phi_M(\mathcal{M},v')$.
\end{axiom}

That is, the value in the combined game is the sum of the values in individual games. This property will imply that valuations of a model for different datapoints, when added, leads to its valuation on the dataset. 

\subsection{The Shapley value}
The Shapley value \cite{shapley1953value} of a classifier is the average marginal contribution of the model over the possible different permutations in which the ensemble can be formed \cite{chalkiadakis2011computational}. It is a solution concept which satisfies Axioms \ref{shapley:ax_1}-\ref{shapley:ax_4} and the only solution concept which is uniquely characterized by Axioms \ref{shapley:ax_3} and \ref{shapley:ax_4}.
\begin{definition} \textbf{Shapley value.}
The Shapley value of binary classifier $M$ in the ensemble $\mathcal{M}$, for the data point level ensemble game $G=(\mathcal{M},v)$ is defined as
{\small$$\Phi_M(v)=\sum\limits_{\mathcal{S} \subseteq \mathcal{M} \setminus \{M\}} \frac{|\mathcal{S}|!\; (|\mathcal{M}|-|\mathcal{S}|-1)!}{|\mathcal{M}|!}(v(\mathcal{S}\cup\{M\})-v(\mathcal{S})).$$}
\end{definition}
Calculating the exact Shapley value for every model in an ensemble game would take $\mathcal{O}(m!)$ time, or more, which is computationally unfeasible in large ensembles. We discuss a range of approximation approaches in detail which can give Shapley value estimates in $\mathcal{O}(m)$ and $O(m^2)$ time.

\subsubsection{Multilinear extension (MLE) approximation of the Shapley value.} The \textit{MLE} approximation of the Shapley value in a voting game \cite{shapley_mle,fatima2008linear} can be used to estimate the Shapley value in the ensemble game and its dual game. Let us define the expectation and the variation of the aggregated ensemble contributions for the remaining models as: $\mu_M=\sum_{i=1}^{m} w_j-w_M$ and
 $\nu_M=\sum_{i=1}^{m}w_j^2-w_M^2$.
For a classifier $M\in\mathcal{M}$ the multi-linear approximation of the unnormalized Shapley value is computed by:

{\footnotesize
\begin{align*}
\widehat{\Phi}_M\propto \int\limits_{-\infty}^{\gamma}\frac{1}{\sqrt{2 \pi \nu_M}} \exp\left(-\frac{(x-\mu_M)^2}{2\nu_M}\right) \mathrm{d} x-\int\limits_{-\infty}^{\gamma-w_M}\frac{1}{\sqrt{2 \pi \nu_M}} \exp\left(-\frac{(x-\mu_M)^2}{2\nu_M}\right) \mathrm{d} x.
\end{align*}}

This approximation assumes that the size of the game is large (many classifiers in the ensemble in our case) and also that $\mu$ has an \textit{approximate normal distribution}. Calculating all of the approximate Shapley values by \textit{MLE} takes $\mathcal{O}(m)$ time.

\subsubsection{Monte Carlo (MC) approximation of the Shapley value.}
The MC approximation \cite{shapley_monte_carlo,maleki} given the ensemble $\mathcal{M}$ estimates the Shapley value of the model $M\in \mathcal{M}$ by the average marginal contribution over uniformly sampled permutations.
{\small\begin{align}
\hat{\Phi}_M &= \mathbb{E}_{\theta \sim \Theta}[v(\mathcal{S}^M_\theta \cup \{M \})-v(\mathcal{S}^M_{\theta})]\label{eq:permutation_approximation}
\end{align}}

In  Equation \eqref{eq:permutation_approximation}, $\Theta$ is a uniform distribution over the $m!$ permutations of the binary classifiers and $\mathcal{S}^M_\theta$ is the subset of models that appear before the classifier $M$ in permutation $\theta$. Approximating the Shapley value requires the generation of $p$ classifier permutations (for a suitable $p$), and marginal contribution calculations with respect to those contributions -- this takes $\mathcal{O}(mp)$ time.

\begin{figure*}[t!]

\begin{tikzpicture}[thick,scale=0.58, every node/.style={scale=0.55}]

    \node (feature_send) at (-3.0, -3.1) {\Huge\textbf{(a)}};
    \node (predicta) at (4.1, -3.1) {\Huge\textbf{(b)}};
    \node (predictb) at (12.5, -3.1) {\Huge\textbf{(c)}};
    \node (predictc) at (17.5, -3.1) {\Huge\textbf{(d)}};
    
    \node (predictd) at (23.65, -3.1) {\Huge\textbf{(e)}};

    \node (esnemble_label) at (4.1, 3.6) {\LARGE\textbf{Ensemble of binary classifiers}};

    \node (data_owner) at (-3.0, 2.5) {\textbf{Data points}}; 

    \node (x_1) at (-3.0,1.5) [dp] {$\textbf{x}_1$};
    \node (x_2) at (-3.0,0.3) [dp] {$\textbf{x}_2$};
    \node (x_dots) at (-3.0, -0.5) {$\vdots$};    
    \node (x_n) at (-3.0,-1.5) [dp] {$\textbf{x}_n$}; 
    
    \node (arr_1) at (-1.7,1.5) {\Huge$ \mathbf{\rightarrow}$};
    \node (arr_2) at (-1.7,0.5) {\Huge$ \mathbf{\rightarrow}$};
    \node (arr_3) at (-1.7, -0.5){\Huge$ \mathbf{\rightarrow}$};
    \node (arr_4) at (-1.7,-1.5){\Huge$ \mathbf{\rightarrow}$};
    
    \node (model_1_label) at (1.5, 2.5) {\textbf{Model owner 1}};
    
    \node (pred1_1) at (1.5,1.5) [scoring] {$P^1_1=P(y_1=1|\mathbf{x}_1,M_1)$};
    \node (pred2_1) at (1.5,0.3) [scoring] {$P^2_1=P(y_2=1|\mathbf{x}_2,M_1)$};
    \node (pred_dots_1) at (1.5, -0.5) {$\vdots$};    
    \node (pred4_1) at (1.5,-1.5) [scoring] {$P^n_1=P(y_n=1|\mathbf{x}_n,M_1)$};

    \node (jump_1) at (4.1,1.5) {$\mathbf{\hdots}$};
    \node (jump_2) at (4.1,0.5) {$\mathbf{\hdots}$};
    \node (jump_3) at (4.1,-0.5) {$\mathbf{\hdots}$};  
    \node (jump_4) at (4.1,-1.5) {{$\mathbf{\hdots}$};};

    \node (model_m_label) at (6.7, 2.5) {\textbf{Model owner \textit{m}}};
    
    \node (pred1_m) at (6.7,1.5) [scoring] {$P^1_m=P(y_1=1|\mathbf{x}_1,M_m)$};
    \node (pred2_m) at (6.7,0.3) [scoring] {$P^2_m=P(y_2=1|\mathbf{x}_2,M_m)$};
    \node (pred_dots_m) at (6.7, -0.5) {$\vdots$};    
    \node (pred4_m) at (6.7,-1.5) [scoring] {{$P^n_m=P(y_n=1|\mathbf{x}_n,M_m)$};};

    \node (arr_1_2) at (10.0,1.5) {\Huge$ \mathbf{\rightarrow}$};
    \node (arr_2_2) at (10.0,0.5) {\Huge$ \mathbf{\rightarrow}$};
    \node (arr_3_2) at (10.0, -0.5){\Huge$ \mathbf{\rightarrow}$};
    \node (arr_4_2) at (10.0,-1.5){\Huge$ \mathbf{\rightarrow}$};
    
    \node (score_data_owner) at (12.5, 2.5) {\textbf{Label and scores}}; 

    \node (score_1) at (12.5,1.5) [score] {$(y_1;[P^1_1,\hdots, P^1_m])$};
    \node (score_2) at (12.5,0.3) [score] {$(y_2;[P^2_1,\hdots, P^2_m])$};
    \node (score_dots) at (12.5, -0.5) {$\vdots$};    
    \node (score_n) at (12.5,-1.5) [score] {$(y_n;[P^n_1,\hdots, P^n_m])$};
    
    \node (arr_1_3) at (15.0,1.5) {\Huge$ \mathbf{\rightarrow}$};
    \node (arr_2_3) at (15.0,0.5) {\Huge$ \mathbf{\rightarrow}$};
    \node (arr_3_3) at (15.0, -0.5){\Huge$ \mathbf{\rightarrow}$};
    \node (arr_4_3) at (15.0,-1.5){\Huge$ \mathbf{\rightarrow}$};
    
    \node (game_data_owner) at (17.5, 2.5) {\textbf{Ensemble games}}; 

    \node (game_1) at (17.5,1.5) [score] {$(\gamma;[\textbf{w}^1_1,\hdots, \textbf{w}^1_m])$};
    \node (game_2) at (17.5,0.3) [score] {$(\gamma;[\textbf{w}^2_1,\hdots, \textbf{w}^2_m])$};
    \node (game_dots) at (17.5, -0.5) {$\vdots$};    
    \node (game_n) at (17.5,-1.5) [score] {$(\gamma;[\textbf{w}^n_1,\hdots, \textbf{w}^n_m])$};
    
    \node (arr_1_4) at (20.0,1.5) {\Huge$ \mathbf{\rightarrow}$};
    \node (arr_2_4) at (20.0,0.5) {\Huge$ \mathbf{\rightarrow}$};
    \node (arr_3_4) at (20.0, -0.5){\Huge$ \mathbf{\rightarrow}$};
    \node (arr_4_4) at (20.0,-1.5){\Huge$ \mathbf{\rightarrow}$};
    
    \node (shapley_data_owner) at (23.65, 2.5) {\textbf{Shapley values}}; 

    \node (shapley_1) at (22.3,1.5) [score2] {$(\hat{\Phi}^{1,+}_1,\hdots,\hat{\Phi}^{1,+}_m)$};
    \node (shapley_2) at (22.3,0.3) [score2] {$(\hat{\Phi}^{2,+}_1,\hdots, \hat{\Phi}^{2,+}_m)$};
    \node (shapley_dots) at (22.3, -0.5) {$\vdots$};    
    \node (shapley_n) at (22.3,-1.5) [score2] {$(\hat{\Phi}^{n,+}_1,\hdots, \hat{\Phi}^{n,+}_m)$};
    
    \node (shapley_1p) at (25.0,1.5) [score2] {$(\hat{\Phi}^{1,-}_1,\hdots,\hat{\Phi}^{1,-}_m)$};
    \node (shapley_2) at (25.0,0.3) [score2] {$(\hat{\Phi}^{2,-}_1,\hdots, \hat{\Phi}^{2,-}_m)$};
    \node (shapley_dotsp) at (25.0, -0.5) {$\vdots$};    
    \node (shapley_np) at (25.0,-1.5) [score2] {$(\hat{\Phi}^{n,-}_1,\hdots, \hat{\Phi}^{n,-}_m)$};

  
    \begin{pgfonlayer}{background}
         \path (model_1_label.west |- model_m_label.north)+(-1.3,1.5) node (a_main) {};
         \path (pred4_1.south -| pred4_m.east)+(+0.7,-0.6) node (b_main) {};
         \path[fill=yellow!20,rounded corners, draw=black!50, dashed]
            (a_main) rectangle (b_main);

        \path (data_owner.north west)+(-0.1,0.3) node (a_0) {};
        \path (x_n.south -| data_owner.east)+(+0.1,-0.3) node (b_0) {};
        \path[fill=red!10,rounded corners, draw=black!50, dashed]
            (a_0) rectangle (b_0);    
            
        \path (model_1_label.north west)+(-0.9,0.3) node (a_1) {};
        \path (pred4_1.south -| model_1_label.east)+(+0.9,-0.3) node (b_1) {};
        \path[fill=blue!10,rounded corners, draw=black!50, dashed]
            (a_1) rectangle (b_1);
            
        \path (model_m_label.north west)+(-0.9,0.3) node (a_m) {};
        \path (pred4_m.south -| model_m_label.east)+(+0.9,-0.3) node (b_m) {};
        \path[fill=blue!10,rounded corners, draw=black!50, dashed]
            (a_m) rectangle (b_m);
            
        \path (score_data_owner.north west)+(-0.5,0.3) node (score_a_0) {};
        \path (score_n.south -| score_data_owner.east)+(+0.5,-0.3) node (score_b_0) {};
        \path[fill=red!10,rounded corners, draw=black!50, dashed]
            (score_a_0) rectangle (score_b_0); 
            
        \path (game_data_owner.north west)+(-0.5,0.3) node (game_a_0) {};
        \path (game_n.south -| game_data_owner.east)+(+0.5,-0.3) node (game_b_0) {};
        \path[fill=red!10,rounded corners, draw=black!50, dashed]
            (game_a_0) rectangle (game_b_0);
            
        \path (shapley_1.north west)+(-0.5,1.1) node (shapley_a_0) {};
        \path (shapley_np.south east)+(+0.5,-0.3) node (shapley_b_0) {};
        \path[fill=red!10,rounded corners, draw=black!50, dashed]
            (shapley_a_0) rectangle (shapley_b_0);
            
    \end{pgfonlayer}
\end{tikzpicture}
\vspace{-7mm}
\caption{The approximate average Shapley value calculation pipeline. (a) Given the set of labeled datapoints the data owner sends the features of data points to the binary classifiers in the ensemble. (b) The models in the the ensemble score each of the data points independently and send the probability scores back to the oracle. (c) Given the label and the scores the oracle defines the ensemble game on each of the data points. (d-e) The ensemble games and the dual ensemble games are solved by the oracle and approximate Shapley values of models are calculated in each of the data point level ensemble games.}\label{fig:shapley_flow_chart}
\end{figure*}
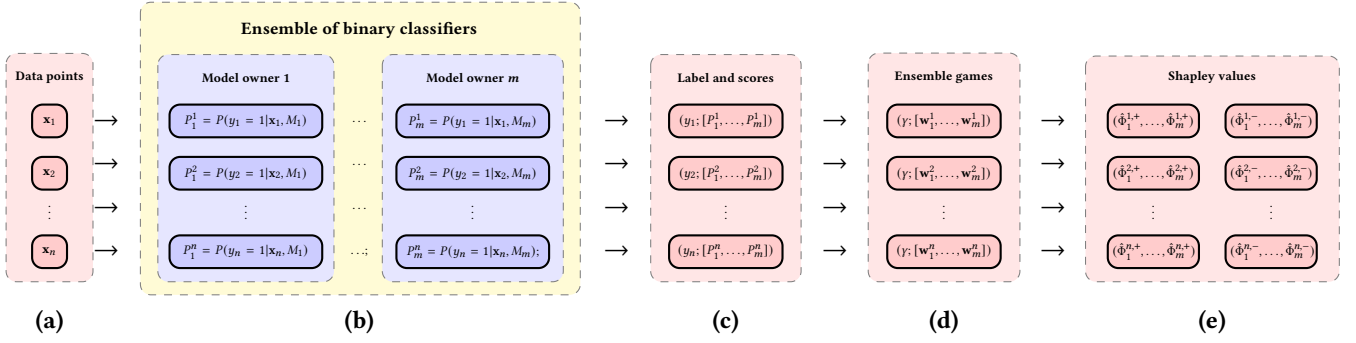

\subsubsection{Voting game approximation of the Shapley value.} The ensemble games introduced above are a variant of voting games \cite{osborne1994course}, hence we can use the \textit{Expected Marginal Contributions (EMC)} approximation \cite{fatima2008linear}. This procedure sums the expected marginal contributions of a model to fixed size ensembles -- it is described in the pseudo-code Algorithm \ref{alg:shapley_approximation_algorithm}. 

The algorithm iterates over the individual model weights and initializes Shapley values as zeros (lines 1-2). For each ensemble size it calculates the expected contribution of the model to the ensemble. This expected contribution is the probability that a classifier becomes the marginal voter. These marginal contributions are added to the approximate Shapley value (lines 3-8). Using this Shapley value approximation technique takes $\mathcal{O}(m^2)$ time. However, it is the most accurate Shapley value approximation technique in terms of absolute error of the Shapley value estimates \cite{fatima2008linear}.
\vspace{-3mm}

\begin{algorithm}[h!]
{\small
		\DontPrintSemicolon
		\SetAlgoLined
		\KwData{
		$[w_1,\dots,w_m]$ -- Weights of binary classifiers.\\
		$\,\,\,\,\,\,\,\,\,\,\,\,\,\,\,\,\,\gamma$ -- Cutoff value.\\
		$\,\,\,\,\,\,\,\,\,\,\,\,\,\,\,\,\,\delta$ -- Numerical stability parameter.\\
		$\,\,\,\,\,\,\,\,\,\,\,\,\,\,\,\,\,\mu$ -- Expected value of weights.\\
		$\,\,\,\,\,\,\,\,\,\,\,\,\,\,\,\,\,\nu$ -- Variance of weights.\\
		}
		\KwResult{$(\widehat{\Phi}_1, \dots, \widehat{\Phi}_m)$  -- Approximate Shapley values.}
			\For{$j \in  \left \{1,\dots ,m\right \}$}{
            $\widehat{\Phi}_j \leftarrow 0$\;
		    \For{$k\in \left \{1,\dots,m-1\right\}$}{
            $a \leftarrow (\gamma-w_j/k$\;
            $b \leftarrow (\gamma-\delta)/k$\;
            $\widehat{\Phi}_j\leftarrow\widehat{\Phi}_j+\frac{1}{\sqrt{2\pi \nu /k}}\int\limits_{a}^{b} \exp(-k\frac{(x-\mu)^2}{
            2\nu})\mathrm{d} x$\;
		    }
			}
			}
		\caption{Expected marginal contribution approximation of Shapley values based on \cite{fatima2008linear}.}\label{alg:shapley_approximation_algorithm}
	\end{algorithm}

\vspace{-5mm}

	\begin{algorithm}[hb!]
	{\small	\DontPrintSemicolon
		\SetAlgoLined
		\KwData{$(\textbf{x},y)$ -- Labeled data point.\\
		$\,\,\,\,\,\,\,\,\,\,\,\,\,\,\,\left\{ M_1,\dots,M_m \right\}$ -- Set of binary classifiers.\\
		}
		\KwResult{$[w_1,\dots,w_m]$ -- Weights of individual models.}

			\For{$j  \in \left \{1,\dots ,m\right \}$}{
            \uIf{$y = 1$}{
            $w_j\leftarrow P(y=1\mid M_j,\textbf{x})/m$\;            
            }
            \Else{
            $w_j\leftarrow P(y=0\mid M_j,\textbf{x})/m$\;             
            }
            }
}
		\caption{Calculating the individual model weights in an ensemble game given a data point and a classifier ensemble.}\label{alg:shapley_scoring_algorithm}
	\end{algorithm}

\section{Model valuation with the Shapley value of ensemble games}\label{sec:shapley_algorithm}

In this section we present  \textit{Troupe}, the mechanism used for model valuation in ensembles. The chart in Figure \ref{fig:shapley_flow_chart} gives a high-level summary of our model valuation approach.

\subsection{The model valuation algorithm}
The Shapley value based model evaluation procedure described by Algorithm \ref{alg:shapley_main_algorithm} uses $\mathcal{D}=\left \{(\textbf{x}_1, y_1); \dots; (\textbf{x}_n, y_n)\right \}$, a set of $n$ datapoints with binary labels, for evaluating the ensemble $\mathcal{M}$ given a cutoff value $0 \leq \gamma\leq 1$ and a numerical stability parameter $\delta\in\mathbb{R}^{0+}$.

We set $m$ the number of models and $n$ the number of labeled data points used in the model valuation (lines 1-2). We iterate over all of the data points (line 3) and given a data point for all of the models we compute the individual model weights in an ensemble game using the label, features and the model itself (line 4) -- see Algorithm \ref{alg:shapley_scoring_algorithm} for the exact details. Using the individual model weights in the ensemble game we calculate the expected value and variance of the data point (lines 5-6).

If the sum of individual  model weights is higher than the cutoff value (line 7) the Shapley values in the ensemble game are approximated by Algorithm \ref{alg:shapley_approximation_algorithm} for the data point (line 8). The Shapley values in the dual ensemble game are defined by a vector of zeros (line 9). If the point is misclassified (line 10) a \textit{dual ensemble game} is defined for the data point.  This requires the redefined cutoff value $\widetilde{\gamma}$, expected value $\widetilde{\mu}$ and variance $\widetilde{\nu}$ of the individual model weight vector (lines 11-13). We exploit the linearity of expectation and the location invariance of the variance. The individual model weight vector is redefined (line 14) in order to quantify the role of models in the erroneous decision. The original ensemble game Shapley values are initialized with zeros (line 15) and using the new parametrization of the game we approximate the dual ensemble game Shapley values (line 16) by Algorithm \ref{alg:shapley_approximation_algorithm}. The algorithm outputs data point level Shapley values for each model.
\clearpage

	\begin{algorithm}[h!]
{\small		\DontPrintSemicolon
		\SetAlgoLined
		\KwData{$\mathcal{D}$ -- Labeled data points.\\
		$\,\,\,\,\,\,\,\,\,\,\,\,\,\,\,\mathcal{M}$ -- Set of binary classifiers.\\
		$\,\,\,\,\,\,\,\,\,\,\,\,\,\,\,\gamma$ -- Cutoff value.\\
		$\,\,\,\,\,\,\,\,\,\,\,\,\,\,\,\delta$ -- Numerical stability parameter.\\
		}
		\KwResult{$\underbrace{\left\{(\widehat{\Phi}^{1,+}_0,\dots,\widehat{\Phi}^{1,+}_m),\dots,(\widehat{\Phi}^{n,+}_0,\dots,\widehat{\Phi}^{n,+}_m) \right\}}_{\text{Shapley value vectors for ensemble games.}}$\\
		$\,\,\,\,\,\,\,\,\,\,\,\,\,\,\,\,\,\,\,\,\underbrace{\left\{(\widehat{\Phi}^{1,-}_0,\dots,\widehat{\Phi}^{1,-}_m),\dots,(\widehat{\Phi}^{n,+}_0,\dots,\widehat{\Phi}^{n,+}_m) \right\}}_{\text{Shapley value vectors for dual ensemble games.}}$}
$m\leftarrow |\mathcal{M}|$\;
$n\leftarrow |\mathcal{D}|$\;
\For{$i \in \left \{1,\dots ,n\right \}$}{
            
			$[w^i_1,\dots,w^i_m]\leftarrow \textbf{Get Weights}((\textbf{x}_i,y_i);\left\{M_1,\dots ,M_m\right\})$\;
            $\mu \leftarrow \sum\limits_{j=1}^{m}w^i_j/m$\;
            $\nu \leftarrow \sum\limits_{j=1}^{m}(w^i_j-\mu)^2/m$\;
            \uIf{$\mu > \gamma/m$}{
            $(\widehat{\Phi}^{i,+}_0,\dots,\widehat{\Phi}^{i,+}_m)\leftarrow \textbf{Shapley}\left([w^i_1,\dots,w^i_m];\gamma;\delta;\mu;\nu\right)$\;
            $(\widehat{\Phi}^{i,-}_0,\dots,\widehat{\Phi}^{i,-}_m)\leftarrow \textbf{0}$\;
            }
            \Else{
            $\widetilde{\gamma} \leftarrow 1- \gamma$ \;
            $\widetilde{\mu} \leftarrow 1/m- \mu$ \;
            $\widetilde{\nu}\leftarrow\nu$\;
            $[\widetilde{w}^i_1,\dots,\widetilde{w}^i_m]\leftarrow [1/m -w^i_1,\dots, 1/m-w^i_m]$\;
            $(\widehat{\Phi}^{i,+}_0,\dots,\widehat{\Phi}^{i,+}_m)\leftarrow \textbf{0}$\;
            $(\widehat{\Phi}^{i,-}_0,\dots,\widehat{\Phi}^{i,-}_m)\leftarrow \textbf{Shapley}\left([\widetilde{w}_1,\dots,\widetilde{w}_m];\widetilde{\gamma};\delta;\widetilde{\mu};\widetilde{\nu}\right)$\; 
            }
}}
		\caption{\textit{Troupe:} Calculating the approximate Shapley value of the classifiers in ensemble and dual ensemble games.}\label{alg:shapley_main_algorithm}

	\end{algorithm}

\subsection{Measuring ensemble heterogeneity}
The data point level Shapley values of ensemble games and their dual can be aggregated to measure the importance of models in ensemble level decisions.

In order to quantify the role of models in classification and misclassification we calculate the average Shapley value of models in the ensemble and dual ensemble games conditional on the success of classification. The sets $\mathcal{N}^+$ and $\mathcal{N}^-$ contain the indices of classified and misclassified data points. Using the cardinality of these sets $n^+=|\mathcal{N}^+|$, $n^-=|\mathcal{N}^-|$  and the data point level Shapley values output by Algorithm \ref{alg:shapley_approximation_algorithm} we can estimate the \textit{conditional} role of models in classification and misclassification by averaging the approximate Shapley values using Equations \eqref{eq:ave_shap_1} and \eqref{eq:ave_shap_2}.

\begin{align}
    (\overline{\Phi}^+_1,\dots, \overline{\Phi}^+_m)&=\sum\limits_{i\in\mathcal{N}^+} (\widehat{\Phi}^{i,+}_0,\dots,\widehat{\Phi}^{i,+}_m)/n^+\label{eq:ave_shap_1}\\
    (\overline{\Phi}^-_1,\dots, \overline{\Phi}^-_m)&=\sum\limits_{i\in\mathcal{N}^-}  (\widehat{\Phi}^{i,-}_0,\dots,\widehat{\Phi}^{i,-}_m)/n^-\label{eq:ave_shap_2}
\end{align}

If a component of the average Shapley value vector in Equation \eqref{eq:ave_shap_1} is large compared to other components the corresponding model has an important role in the correct classification decisions of the ensemble. A large component in Equation \eqref{eq:ave_shap_2} corresponds to a model which is responsible for a large number of misclassifications.



\subsection{Theoretical properties}
Our framework utilizes labeled data instances to approximate the importance of classifiers in the ensemble and this has important implications. In the following we discuss how the size of the dataset and the number classifiers affects the Shapley value approximation error and the runtime of \textit{Troupe}.
\subsubsection{Bounding the average approximation error.} Our discussion focuses on the average conditional Shapley value in ensemble games. However, analogous results can be obtained for the Shapley values computed from dual ensemble games. 
\begin{definition} \textbf{Approximation error.} The Shapley value approximation error of model $M \in \mathcal{M}$ in an ensemble game is defined as $\Delta \Phi^+_M =\widehat{\Phi}^+_M-\Phi^+_M$.
\end{definition}

\begin{definition} \textbf{Average approximation error.} 
Let use denote the Shapley value approximation errors of model $M \in \mathcal{M}$ calculated from the dataset $\mathcal{D}$ of correctly classified points as $\Delta \Phi^{1,+}_M,\dots,\Delta \Phi^{n,+}_M$. The average approximation error is defined by $\overline{\Delta \Phi}^+_M=\sum_{i=1}^{n} \Delta \Phi^{i,+}_M/n$.
\end{definition}
\begin{theorem}\label{thm:main_bound}\textbf{Average conditional Shapley value error bound.} 
If the Shapley value approximation errors $\Delta \Phi^{1,+}_M,\dots,\Delta \Phi^{n,+}_M$ of model $M \in \mathcal{M}$ calculated by Algorithm \ref{alg:shapley_main_algorithm} from the dataset $\mathcal{D}$ are independent random variables than for any $\varepsilon\in \mathbb{R}^+$ Inequality \eqref{eq:main_theorem} holds.

\begin{align}
P(|\overline{\Delta \Phi}^+_M-\mathbb{E}[\overline{\Delta \Phi}^+_M]|\geq \varepsilon )&\leq 2\exp \left(-\sqrt{\frac{n^2\cdot m\cdot \varepsilon^4 \cdot \pi}{8}}\right)\label{eq:main_theorem}
\end{align}
\end{theorem}
\begin{proof} Let us first note the fact that every absolute Shapley approximation value is bounded by the inequality described in Lemma \ref{lemma:shap_bound}. 
\begin{lemma}\label{lemma:shap_bound} \textbf{Approximate Shapley value bound.} As Theorem (10) of \cite{fatima2008linear} states the approximation error of the Shapley value in a single voting game is bounded by Inequality \eqref{eq:fatima_inequality} when the expected marginal contributions approximation is used.
\begin{align}
-\sqrt{\frac{8}{m\pi}}\leq\Delta \Phi^+_M\leq\sqrt{\frac{8}{m\pi}}\label{eq:fatima_inequality}
\end{align}
\end{lemma}

Using Lemma \ref{lemma:shap_bound}, the fact that \textit{Troupe} is based on the expected marginal contributions approximation and that the Shapley values of a model in different ensemble games are independent random variables we can use Hoeffding's second inequality \cite{hoeffding} for bounded non zero-mean random variables:

\begin{align*}
P(|\overline{\Delta \Phi}^+_M-\mathbb{E}[\overline{\Delta \Phi}^+_M]|\geq \varepsilon )\leq& 2\exp\left (-\frac{2\varepsilon^2 n^2}{\sum \limits_{i=1}^{n}\left [ \left(\sqrt{\frac{8}{m\pi}}\right)-\left(-\sqrt{\frac{8}{m\pi}}\right)\right]} \right)\\
P(|\overline{\Delta \Phi}^+_M-\mathbb{E}[\overline{\Delta \Phi}^+_M]|\geq \varepsilon )\leq& 2\exp\left (-\frac{2\varepsilon^2 n^2}{2n\sqrt{\frac{8}{m\pi}}} \right)
\end{align*}

\end{proof}
\begin{theorem}\label{thm:confidence}\textbf{Confidence interval of the expected average approximation error.}
In order to acquire an $(1-\alpha)$-confidence interval of $\mathbb{E}[\overline{\Delta \Phi }^+_M]\pm \varepsilon$ one needs a labeled dataset $\mathcal{D}$ of correctly classified data points for which $n$ the cardinality of $\mathcal{D}$, satisfies Inequality \eqref{eq:conf}.

\begin{align}
n&\geq \sqrt{\frac{8\ln ^2 \left(\frac{\alpha}{2}\right)}{\varepsilon^4 m\pi}}\label{eq:conf}
\end{align}

\end{theorem}
\begin{proof}
The probability $P(|\overline{\Delta \Phi}^+_M-\mathbb{E}[\overline{\Delta \Phi}^+_M]|\geq \varepsilon )$ in Theorem \ref{thm:main_bound} equals to the level of significance for the confidence interval $\mathbb{E}[\overline{\Delta \Phi_M}^+]\pm \varepsilon$. Which means that Inequality \eqref{eq:alpha} holds for the significance level $\alpha$.
\begin{align}
\alpha\leq 2\exp \left(-\sqrt{\frac{n^2\cdot m\cdot \varepsilon^4 \cdot \pi}{8}}\right).\label{eq:alpha}
\end{align}
Solving inequality \eqref{eq:alpha} for $n$ yields the cardinality of the dataset (number of correctly classified data points) required for obtaining the confidence interval described in Theorem \ref{thm:confidence}.
\end{proof}

The inequality presented in Theorem \ref{thm:confidence} has two important consequences regarding the bound:
\begin{enumerate}
\item Larger ensembles require less data in order to give confident estimates of the Shapley value for individual models.
\item The dataset size requirement is sublinear in terms of confidence level and quadratic in the precision of the Shapley value approximation. 
\end{enumerate}
\subsubsection{Runtime and memory complexity.} The runtime and memory complexity of the proposed model valuation framework depends on the complexity of the main evaluation phases. We assume that our framework operates in a single-core non distributed setting.

\textit{Scoring and game definition.} Assuming that the scoring of a data point takes $\mathcal{O}(1)$ time, scoring the data point with all models takes $\mathcal{O}(m)$. Scoring the whole dataset and defining games both takes $\mathcal{O}(nm)$ time and $\mathcal{O}(nm)$ space respectively.

\textit{Approximation and overall complexity.} Calculating the expected marginal contribution of a model to a fixed size ensemble takes $\mathcal{O}(1)$ time. Doing this for all of the ensemble sizes takes $O(m)$ time. Approximating the Shapley value for all models requires $O(m^2)$ time and $O(m)$ space. Given a dataset of $n$ points this implies a need for $O(nm^2)$ time and $O(nm)$ space. This is also the overall time and space complexity of the proposed framework.
\section{Experimental evaluation}\label{sec:shapley_experiments}
In this section, we show that \textit{Troupe} approximates the average of Shapley values precisely. We provide evidence that Shapley values are a useful decision metric for ensemble creation. Our results illustrate that model importance and complexity are correlated, and that Shapley values are able to identify adversarial models in the ensemble.  Our evaluation is based on various real world binary graph classification tasks \cite{karateclub}. Specifically, we use datasets collected from \textit{Reddit}, \textit{Twitch} and \textit{GitHub} -- the descriptive statistics of these datasets are in Table \ref{tab:shapley_datasets}.

\begin{table}[h!]
	\centering

	\caption{Descriptive statistics of the  graph classification datasets taken from \cite{karateclub} used for the evaluation of \textit{Troupe}.} 
	\label{tab:shapley_datasets}
	
	{\footnotesize
\setlength\tabcolsep{4pt} 

\begin{tabular}{ccccccccc}
            &\multicolumn{2}{c}{\textbf{Classes}} & \multicolumn{2}{c}{\textbf{Nodes}} & \multicolumn{2}{c}{\textbf{Density}} & \multicolumn{2}{c}{\textbf{Diameter}} \\[0.25em]
            \cline{2-3} \cline{4-5}\cline{6-7}\cline{8-9}   
\textbf{Dataset}     & \textbf{Positive} &\textbf{Negative} &  \textbf{Min}         & \textbf{Max}         & \textbf{Min}          & \textbf{Max}          & \textbf{Min}& \textbf{Max}          \\[0.25em]\hline
\textbf{Reddit} & 521  &479 & 11 &93 &0.023  &0.027  &2             &18   \\[0.3em]
\textbf{Twitch}   &520&480&14&52&0.039&0.714&2&2 \\[0.3em]
\textbf{GitHub}   &552&448      &   10    &  942    &  0.004  &0.509  &   2 &     15     \\[0.3em]
			\hline
		\end{tabular}}

	\end{table}

\subsection{The precision of approximation}\label{appendix:approximation_experiment}
The Shapley value approximation performance of \textit{Troupe} is compared to that of various other estimation schemes \cite{maleki,shapley_mle} discussed earlier. We use an ensemble of logistic regressions where each classifier is trained on features extracted with a whole graph embedding technique \cite{feather,netlsd,fgsd}. We utilize 50\% of the graphs for training and calculated the average conditional Shapley values of ensemble games and dual ensemble games using the remaining 50\% of the data. \subsubsection{Experimental details}
The features of the graphs are extracted with whole graph embedding techniques implemented in the open source \textit{Karate Club} framework \cite{karateclub}. Given a set of graphs $\mathcal{G}=(G_1,\dots, G_n)$ whole graph embedding algorithms \cite{ netlsd, feather} learn a mapping $g: \mathcal{G}\to \mathbb{R}^d$ which delineate the graphs $G \in \mathcal{G}$ to a $d$ dimensional metric space. We utilize the following whole graph embedding and statistical fingerprinting techniques:

\begin{enumerate}
    \item \textbf{FEATHER} \cite{feather} uses the characteristic function of topological features as a graph level statistical descriptor.
    \item \textbf{Graph2Vec} \cite{graph2vec} extracts tree features from the graph.
    \item \textbf{GL2Vec} \cite{gl2vec} distills tree features from the dual graph.
    \item \textbf{NetLSD} \cite{netlsd} derives characteristics of graphs using the heat trace of the graph spectra.
    \item \textbf{SF} \cite{sf} utilizes the largest eigenvalues of the graph Laplacian matrix as an embedding.
    \item \textbf{LDP} \cite{ldp} sketches the histogram of local degree distributions.
    \item \textbf{GeoScattering} \cite{geoscattering} applies the scattering transform to various structural features (e.g. degree centrality).
    \item \textbf{IGE} \cite{ige_icml} combines graph features from local degree distributions and scattering transforms.
    \item \textbf{FGSD} \cite{fgsd} sketches the Moore-Penrose spectrum of the normalized graph Laplacian with a histogram.
\end{enumerate}

The embedding techniques use the \textit{default settings} of the \textit{Karate Club} library, each embedding dimension is column normalized and the graph features are fed to the \textit{scikit-learn} implementation of logistic regression. This classifier is trained with  $\ell_2$ penalty cost, we choose an SGD optimizer and the regularization coefficient $\lambda$ was set to be $10^{-2}$. 

\subsubsection{Experimental findings}
In Table \ref{tab:shapley_approximation} we summarize the absolute percentage error values (compared to exact Shapley values in ensemble games) obtained with the various approximations schemes. (i) Our empirical results support that \textit{Troupe} consistently computes accurate estimates of the ground truth model evaluations across datasets and classifiers in the ensemble. (ii) The high quality of estimates suggests that the approximate Shapley values of models extracted by \textit{Troupe} can serve as a proxy decision metric for ensemble building and model selection.



\begin{table*}[t!]
\caption{Absolute percentage error of average conditional Shapley values obtained by approximation techniques (rows) for the graph classifiers (columns) in the ensemble game. Bold numbers note the lowest error on each dataset -- classifier pair.}\label{tab:shapley_approximation}
	{\small\centering
\begin{tabular}{ccccccccccc}
                         & \textbf{Approximation}     & \textbf{FEATHER} & \textbf{Graph2Vec} & \textbf{GL2Vec} & \textbf{NetLSD} & \textbf{SF} & \textbf{LDP} & \textbf{GeoScatter} & \textbf{IGE} & \textbf{FGSD} \\ \hline
   \multirow{4}{*}{\textbf{Reddit}}& \textbf{Troupe}& \textbf{1.23}&\textbf{2.35}&8.18& \textbf{0.99}&\textbf{2.64}&\textbf{2.31}&\textbf{1.64}&\textbf{4.85}&\textbf{1.49}\\
                                   & \textbf{MLE}  & 3.20&23.61&32.62&4.19&5.34&7.97&7.12&5.42&7.61\\
                                   &$\textbf{MC}$ ${p=10^3}$ & 12.57&30.94&13.26&8.67&32.76&12.32& 11.62&16.36&12.78\\
                                   & $\textbf{MC}$ ${p=10^3}$  & 4.71&5.38&\textbf{3.41}&1.67&3.03&5.51&4.34&4.97&3.82\\
                         \hline
\multirow{4}{*}{\textbf{Twitch}} & \textbf{Troupe}& \textbf{0.28}&\textbf{3.33}&\textbf{1.18}&\textbf{2.53}&\textbf{1.62}&\textbf{0.59}&\textbf{1.48}&\textbf{0.25}&1.19\\
                         & \textbf{MLE}  & 5.22&5.44&3.05&8.32&2.38&1.92&3.14&4.85&5.77\\
                         &  $\textbf{MC}$ ${p=10^2}$& 2.37&10.40&6.76&7.07&15.79&6.36&13.99&23.96&\textbf{0.39}\\
                         &  $\textbf{MC}$ ${p=10^3}$ & 2.32&4.60&2.96&2.67&2.53&2.73&6.31&0.27&3.89\\
                         \hline
\multirow{4}{*}{\textbf{GitHub}} &\textbf{Troupe}&  \textbf{2.68}&\textbf{0.18}&\textbf{2.61}&\textbf{1.41}&\textbf{1.49}&1.23&\textbf{1.88}&\textbf{2.36}&1.04\\
                         & \textbf{MLE}  & 9.22&5.12&3.76&3.27&9.71&7.46&5.08&4.04&0.73\\
                         & $\textbf{MC}$ ${p=10^2}$& 5.91&9.37&4.67&9.82&8.78&8.34&13.66&28.95&\textbf{0.76}\\
                         & $\textbf{MC}$ ${p=10^3}$     & 3.35&6.09&7.70&3.26&2.84&\textbf{0.79}&6.67&2.51&1.12\\\hline
\end{tabular}}
\end{table*}

\subsection{Ensemble building}\label{subsec:building}
Our earlier results suggested that the approximate Shapley values output by \textit{Troupe} (and other estimation methods) can be used as decision metrics for ensemble building. 
\input{figures/ensemble_creation.tex}

\subsubsection{Experimental settings}
We demonstrate this by selecting a high performance subset of a random forest in a forward fashion using the estimated model valuation scores. From each graph we extract Weisfeiler-Lehman tree features \cite{graph2vec} and keep those topological patterns which appeared in at least 5 graphs. Utilizing the counts of these features in graphs we define statistical descriptors. The selection procedure and evaluation steps are the following:
\begin{enumerate}
\item Using 40\% of the graphs we train a random forest with 50 classification trees, each tree is trained on 20 randomly sampled Weisfeiler-Lehman features -- we use the default settings of \textit{scikit-learn}. 
\item We calculate the average conditional Shapley value of classifiers in the games using 30\% of the data.
\item We order the classifiers by the approximate Shapley values in decreasing order, create subensembles in a forward fashion and calculate the predictive performance of the resulting subensembles on the remaining 30\% of the graphs. 
\end{enumerate}
The test performance (measured by AUC scores) of these classifiers and baselines as a function of ensemble size is plotted on Figure \ref{fig:pruning_experiment}. 

\subsubsection{Experimental findings} Our results suggest that Troupe has a material advantage over competing Shapley value approximation schemes. Moreover, the proposed method outperforms the selection of ensemble building baselines on the Twitch and Reddit datasets \cite{zhang2006ensemble,zhang2011sparse,li2012diversity,margineantu1997pruning}. This implies that Troupe is a good alternative for existing ensemble pruning techniques and our choice of the Shapley value approximation scheme is justified in practical applications.
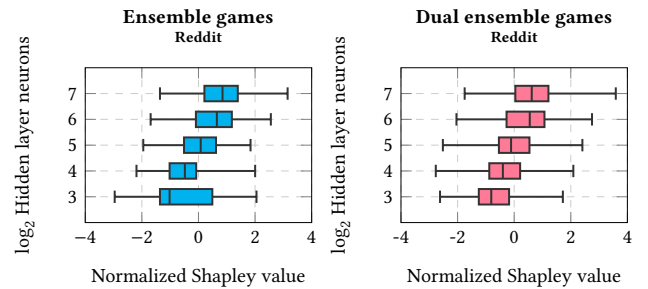
\begin{figure}[h!]
	\centering
	\begin{tikzpicture}[scale=0.95,transform shape]
	\tikzset{font={\fontsize{8pt}{8}\selectfont}}
\begin{groupplot}[
	grid=major,
	grid style={dashed, gray!40},
group style={
                      group name=myplot,
                      group size= 2 by 2, horizontal sep=1.25cm,vertical sep=1.95cm},height=3.75cm,width=4.75cm, title style={at={(0.5,1.0)},anchor=south},
                      ylabel style={at={(0.2,0.5)},anchor=south},]
\nextgroupplot[
    title=$\underset{\textbf{Reddit}}{\textbf{Ensemble games}}$,
    ytick={1,2,3,4,5},
    ylabel=$\log_2$ Hidden layer neurons ,
    xlabel=Normalized Shapley value,
    yticklabels={3, 4, 5, 6, 7},
    ymin=0,
    ymax=6,
    xmin=-4,
    xmax=4,
    xtick={-4,-2,0,2,4}
 ]
    \addplot+[thick,
    black!80,
    fill=cyan!90,mark={diamond*},
    boxplot prepared={
    box extend=0.6, 
      median=-1.02,
      upper quartile=0.49,
      lower quartile=-1.36,
      upper whisker=2.05,
      lower whisker=-2.96
    },
    ] coordinates {};
    \addplot+[thick,
    black!80,
    fill=cyan!90,mark={diamond*},
    boxplot prepared={
            box extend=0.6, 
      median=-0.48,
      upper quartile=-0.08,
      lower quartile=-1.02,
      upper whisker=2.003,
      lower whisker=-2.19
    },
    ] coordinates {};
    \addplot+[thick,
    black!80,
    fill=cyan!90,mark={diamond*},
    boxplot prepared={
        box extend=0.6, 
      median=0.08,
      upper quartile=0.62,
      lower quartile=-0.51,
      upper whisker=1.84,
      lower whisker=-1.95
    },
    ] coordinates {};
    \addplot+[thick,
    black!80,
    fill=cyan!90,mark={diamond*},
    boxplot prepared={
        box extend=0.6, 
      median=0.65,
      upper quartile=1.18,
      lower quartile=-0.09,
      upper whisker=2.558,
      lower whisker=-1.69
    },
    ] coordinates {};
    \addplot+[thick,
    black!80,
    fill=cyan!90,mark={diamond*},
    boxplot prepared={
        box extend=0.6, 
      median=0.85,
      upper quartile=1.39,
      lower quartile=0.21,
      upper whisker=3.15,
      lower whisker=-1.36
    },
    ] coordinates {};

\nextgroupplot[
    title=$\underset{\textbf{Reddit}}{\textbf{Dual ensemble games}}$,
    ytick={1,2,3,4,5},
    ylabel=$\log_2$ Hidden layer neurons ,
    xlabel=Normalized Shapley value,
    yticklabels={3, 4, 5, 6,7},
    ymin=0,
    ymax=6,
    xmin=-4,
    xmax=4,
    xtick={-4,-2,0,2,4},
    xticklabels={-4,-2,0,2,4}
 ]
    \addplot+[thick,
    black!80,
    fill=awesome!60,mark={diamond*},
    boxplot prepared={
        box extend=0.6, 
      median=-0.81,
      upper quartile=-0.18,
      lower quartile=-1.25,
      upper whisker=1.72,
      lower whisker=-2.62
    },
    ] coordinates {};
    \addplot+[thick,
    black!80,
    fill=awesome!60,mark={diamond*},
    boxplot prepared={
        box extend=0.6, 
      median=-0.40,
      upper quartile=0.21,
      lower quartile=-0.88,
      upper whisker=2.09,
      lower whisker=-2.77
    },
    ] coordinates {};
    \addplot+[thick,
    black!80,
    fill=awesome!60,mark={diamond*},
    boxplot prepared={
        box extend=0.6, 
      median=-0.11,
      upper quartile=0.54,
      lower quartile=-0.54,
      upper whisker=2.41,
      lower whisker=-2.52
    },
    ] coordinates {};
    \addplot+[thick,
    black!80,
    fill=awesome!60,mark={diamond*},
    boxplot prepared={
    box extend=0.6, 
      median=0.55,
      upper quartile=1.07,
      lower quartile=-0.27,
      upper whisker=2.75,
      lower whisker=-2.04
    },
    ] coordinates {};
    \addplot+[thick,
    black!80,
    fill=awesome!60,mark={diamond*},
    boxplot prepared={
    box extend=0.6, 
      median=0.62,
      upper quartile=1.21,
      lower quartile=0.04,
      upper whisker=3.59,
      lower whisker=-1.75
    },
    ] coordinates {};

\end{groupplot}
\end{tikzpicture}
\caption{The distribution of normalized Shapley values for neural networks in ensemble and dual ensemble games conditional on the number of neurons.}\label{fig:shapley_reddit_neural_complexity}
\end{figure}
\vspace{-5mm}

\subsection{Model complexity and influence}\label{subsec:influence}
The Shapley value gives an implicit measure of model influence on the ensemble. It is interesting to see whether there is a connection between relative importance and model complexity (e.g. depth of trees or number of neurons). In order to investigate this, we create a voting expert which consists of neural networks with heterogeneous model complexity. 

\subsubsection{Experimental settings.} We create an ensemble of $m=10^3$ neural networks using \textit{scikit-learn} -- each of these has a single hidden layer. Each model receives 20 randomly selected frequency features as input and has a randomly chosen number of hidden layer neurons -- we uniformly sample this hyperparameter from $\left \{2^3, 2^4, 2^5, 2^6, 2^7\right\}$. Individual neural networks are trained by minimizing the binary cross-entropy with SGD for 200 epochs with a learning rate of $10^{-2}$.

\subsubsection{Experimental findings.} The distribution of normalized average Shapley values for the Reddit dataset are plotted on Figure \ref{fig:shapley_reddit_neural_complexity} for the ensemble and dual ensemble games conditioned on the number of hidden layer neurons. The results imply that more complex models with a larger number of free parameters receive higher Shapley values in both classes of games. In simple terms complex models contribute to correct and incorrect classification decisions at a disproportionate rate.

\subsection{Identifying adversarial models}
Ensembles can be formed by the model owners submitting their classifiers voluntarily to a machine learning market. We investigated how adversarial behaviour of model owners affects the Shapley values of classifiers.
\input{figures/adversarial.tex}
\subsubsection{Experimental settings}
We use the Weisfeiler-Lehman tree features described in Subsections \ref{subsec:building} and \ref{subsec:influence} to train a random forest ensemble of 20 classifiers using 50\% of the data -- each of these models utilizes a random subset of 20 features and had a maximal tree depth of 4. The Shapley values are calculated from the remaining 50\% of the data using \textit{Troupe}. We artificially corrupt the predictions for 10 of the classification trees by mixing the outputted probability values with noise that has $\mathcal{U}(0,1)$ distribution. The corrupt predictions are a convex combination of the original prediction and the noise -- we call the weight of the noise as the adversarial noise ratio. Given an adversarial noise ratio we calculate the mean Shapley value (with standard errors) for the adversarial and normally behaving classification trees in the ensemble.
\subsubsection{Experimental findings}
In Figure \ref{fig:shapley_adversarial} we plotted the mean Shapley value of models which are adversarial and which are not in a scenario where half of the model owners mixed their predictions with noise.  Even with a negligible amount of adversarial noise the Shapley values of adversarial models drop in the ensemble games considerably. This implies that our method can be used to find adversarial models in a community contributed classification ensemble.
\subsection{Scalability}
We plotted the mean runtime of Shapley value approximations calculated from 10 experimental runs for an ensemble with $m=2^5$ and dataset with $n=2^8$ on Figure \ref{fig:shapley_runtime}. All results are obtained with our open-source Shapley value approximation framework. These average runtimes of the approximation techniques are in line with the known and new theoretical results discussed in Sections \ref{sec:shapley_related_work} and \ref{sec:shapley_algorithm}. The precise estimates of Shapley values obtained with \textit{Troupe} come at a runtime cost.
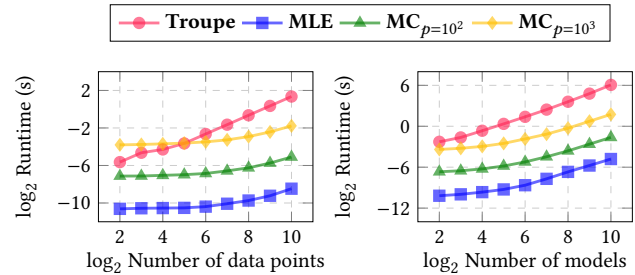
\begin{figure}[h!]
\centering
\scalebox{0.9}{
\begin{tikzpicture}
\begin{groupplot}[	grid=major,
	grid style={dashed, gray!40},group style={
                      group name=myplot,
                      group size= 2 by 1, horizontal sep=1.55cm,vertical sep=1.2cm},height=3.8cm,width=4.75cm, title style={at={(0.5,0.9)},anchor=south},every axis x label/.style={at={(axis description cs:0.5,-0.15)},anchor=north},]
\nextgroupplot[
 	legend columns=4,
	legend style={at={(1.22,1.20)},anchor=south},
y label style={at={(0.05,0.5)}},
    legend entries={\textbf{Troupe}, \textbf{MLE}, $\textbf{MC}_{p=10^2}$,$\textbf{MC}_{p=10^3}$},
	ylabel=$\log_2$ Runtime (s),
	xlabel=$\log_2$ Number of data points,
	xtick={2,4,6,8,10},
	xmin=1,
	xmax=11,
	ymin=-12,
	ymax=4,
	ytick={-10,-6,-2,2},
]
\addplot [very thick, awesome,mark=*,opacity=0.6]coordinates {

(2.0,-5.6406)
(3.0,-4.6437)
(4.0,-4.2912)
(5.0,-3.6482)
(6.0,-2.6423)
(7.0,-1.6521)
(8.0,-0.6482)
(9.0,0.3501)
(10.0,1.3513)
};
\addplot [very thick, blue,mark=square*,opacity=0.6]coordinates {
(2,-10.6237)
(3,-10.5623)
(4,-10.5417)
(5,-10.5195)
(6,-10.3877)
(7,-10.0785)
(8,-9.7373)
(9,-9.2299)
(10,-8.4756)
};
\addplot [very thick, ao(english), ,mark=triangle*,opacity=0.6]coordinates {
(2,-7.1298)
(3,-7.121)
(4,-7.0515)
(5,-6.9816)
(6,-6.8517)
(7,-6.5757)
(8,-6.2506)
(9,-5.7438)
(10,-5.0985)
};
\addplot [very thick, amber,mark=diamond*,opacity=0.6]coordinates {
(2,-3.8026)
(3,-3.7607)
(4,-3.6943)
(5,-3.6158)
(6,-3.4836)
(7,-3.2703)
(8,-2.9211)
(9,-2.4248)
(10,-1.7747)
};

\nextgroupplot[
	xlabel=$\log_2$ Number of models,
	xtick={2,4,6,8,10},
    y label style={at={(0.05,0.5)}},
	ylabel=$\log_2$ Runtime (s),
	xmin=1,
	xmax=11,
	ymin=-14,
	ymax=8,
	ytick={-12,-6,0,6},
	]
\addplot [very thick, awesome,mark=*,opacity=0.6]coordinates {
(2,-2.2801)
(3,-1.6396)
(4,-0.6585)
(5,0.3483)
(6,1.3801)
(7,2.4398)
(8,3.595)
(9,4.7825)
(10,6.0631)
};
\addplot [very thick, blue,mark=square*,opacity=0.6]coordinates {
(2,-10.1865)
(3,-9.9802)
(4,-9.6574)
(5,-9.2614)
(6,-8.6427)
(7,-7.7071)
(8,-6.6881)
(9,-5.7577)
(10,-4.8027)
};
\addplot [very thick, ao(english), ,mark=triangle*,opacity=0.6]coordinates {
(2,-6.672)
(3,-6.525)
(4,-6.2536)
(5,-5.8273)
(6,-5.2309)
(7,-4.473)
(8,-3.61)
(9,-2.6273)
(10,-1.6316)
};
\addplot [very thick, amber,mark=diamond*,opacity=0.6]coordinates {
(2,-3.3989)
(3,-3.2313)
(4,-2.9529)
(5,-2.5128)
(6,-1.8541)
(7,-1.1595)
(8,-0.2797)
(9,0.732)
(10,1.7409)
};

\end{groupplot}
\end{tikzpicture}}
\caption{The runtime of Shapley value approximation in ensemble games as a function of dataset size and number of classifiers in the ensemble.}\label{fig:shapley_runtime}
\end{figure}\vspace{-5mm}

\section{Conclusions and future directions}\label{sec:shapley_conclusions}

We proposed a new class of cooperative games called \textit{ensemble games} in which binary classifiers cooperate in order to classify a data point correctly. We postulated that solving these games with the Shapley value results in a measure of individual classifier quality. We designed \textit{Troupe} a voting game inspired approximation algorithm which computes the average of Shapley values for every classifier in the ensemble based on a dataset. We provided theoretical results about the sample size needed for precise estimates of the model quality.

We have demonstrated that our algorithm can provide accurate estimates of the Shapley value on real world social network data. We illustrated how the Shapley values of the models can be used to create small sized but highly accurate graph classification ensembles. We presented evidence that complex models have an important role in the classification decisions of ensembles. We showcased that our framework can identify adversarial models in the classification ensemble.

We think that our contribution opens up venues for novel theoretical and applied data mining research about ensembles. We theorize that our model valuation framework can be generalized to ensembles which consist of multi-class predictors using a one versus many approach. Our work provides an opportunity for the design and implementation of real world machine learning markets where the payoff of a model owner is  a function of the individual model value in the ensemble.

\clearpage

\bibliographystyle{ACM-Reference-Format}

\bibliography{main}



\end{document}